\documentclass[10pt,twocolumn,letterpaper]{article}

\usepackage{cvpr}              % To produce the CAMERA-READY version
\usepackage{amssymb,amsmath,amsthm,algorithm}
\usepackage{algpseudocode}
\newtheorem{theorem}{Theorem}

\newtheorem{definition}{Definition}

\usepackage{mathtools}
\usepackage{booktabs}
\usepackage{xcolor}
\usepackage{soul}
\usepackage{multirow}
\usepackage{pgfplots} % <<< ADDED for plotting
\pgfplotsset{compat=1.18} % <<< ADDED for modern pgfplots features

\definecolor{cvprblue}{rgb}{0.21,0.49,0.74}
\usepackage[pagebackref,breaklinks,colorlinks,allcolors=cvprblue]{hyperref}

\def\confName{CVPR}
\def\confYear{2025}

\title{ S\textsuperscript{2}D: Selective Spectral Decay for \\ Quantization-Friendly Conditioning of Neural Activations}

\author{
Arnav Chavan\textsuperscript{1}\thanks{Equal contribution. First co-author order has been decided by a coin toss.} \qquad
Nahush Lele\textsuperscript{1}\footnotemark[1] \qquad
Udbhav Bamba\textsuperscript{1}\footnotemark[1] \\
Sankalp Dayal\textsuperscript{1} \qquad
Aditi Raghunathan\textsuperscript{1,2} \qquad
Deepak Gupta\textsuperscript{1} \\[0.3em]
\textbf{\textsuperscript{1}Amazon  \qquad \textsuperscript{2}Carnegie Mellon University}
}

\begin{document}
\maketitle
\begin{abstract}

Activation outliers in large-scale transformer models pose a fundamental challenge to model quantization, creating excessively large ranges that cause severe accuracy drops during quantization. We empirically observe that outlier severity intensifies with pre-training scale (e.g., progressing from CLIP to the more extensively trained SigLIP and SigLIP2). Through theoretical analysis as well as empirical correlation studies, we establish the direct link between these activation outliers and dominant singular values of the weights. Building on this insight, we propose Selective Spectral Decay ($S^2D$), a geometrically-principled conditioning method that surgically regularizes only the weight components corresponding to the largest singular values during fine-tuning. Through extensive experiments, we demonstrate that $S^2D$ significantly reduces activation outliers and produces well-conditioned representations that are inherently quantization-friendly. Models trained with $S^2D$ achieve up to 7\% improved PTQ accuracy on ImageNet under W4A4 quantization and 4\% gains when combined with QAT. These improvements also generalize across downstream tasks and vision-language models, enabling the scaling of increasingly large and rigorously trained models without sacrificing deployment efficiency.
\end{abstract}    
\section{Introduction}
\label{sec:intro}

% \ar{Is this paragraph talking about quantization, or finetuning? Maybe it should solely focus on outliers as a concept} The ``pre-train, then fine-tune" paradigm has become the cornerstone of modern machine learning, enabling large-scale models to achieve state-of-the-art performance across a vast array of tasks. However, this success comes with a significant \ar{this sentence is too vague}``pre-training tax": the emergence of extreme activation outliers that severely hinder the deployment of these models. These outliers, which are rare but high-magnitude values in the activation tensors, are a critical barrier to post-training quantization (PTQ), an essential technique for model compression. The presence of even a single outlier catastrophically inflates the quantization scaling factor, forcing the vast majority of normal activation values to be rounded to zero and causing a massive loss of information. This renders low-bit quantization practically unusable for many models pre-trained with standard techniques. Crucially, these outliers are not an intrinsic property of the models themselves but are increasingly understood to be artifacts of the training strategies employed.

Modern transformer models exhibit an increasingly prominent phenomenon: \emph{activation outliers}, or extremely large values in specific dimensions of neural network activations. These outliers, which can be orders of magnitude larger than typical activation values, occur more severely as models undergo more extensive pre-training \cite{bondarenko-etal-2021-understanding}. Although initially observed primarily in large language models, recent evidence shows this pattern extends broadly across model families and architectures \cite{darcet2024visiontransformersneedregisters}. Activation outliers can severely degrade affine quantization performance, due to inefficient bit allocation. 

Understanding the nature of these outliers is essential before developing mitigation strategies. Outliers severely compromise quantization by inflating activation ranges. For example, a single extreme value can force nearly all activations close to zero to be allotted in the same quantization bin. One could argue that outliers are functionally necessary features essential to the representation space, and removing them will be detrimental to model capability. However recent research on orthogonal optimizers \cite{liu2025muonscalablellmtraining} shows that outliers are an artifact of AdamW's biased optimization \cite{chanko2024adam}.

\begin{figure}[t]
    \centering
    \includegraphics[width=\linewidth]{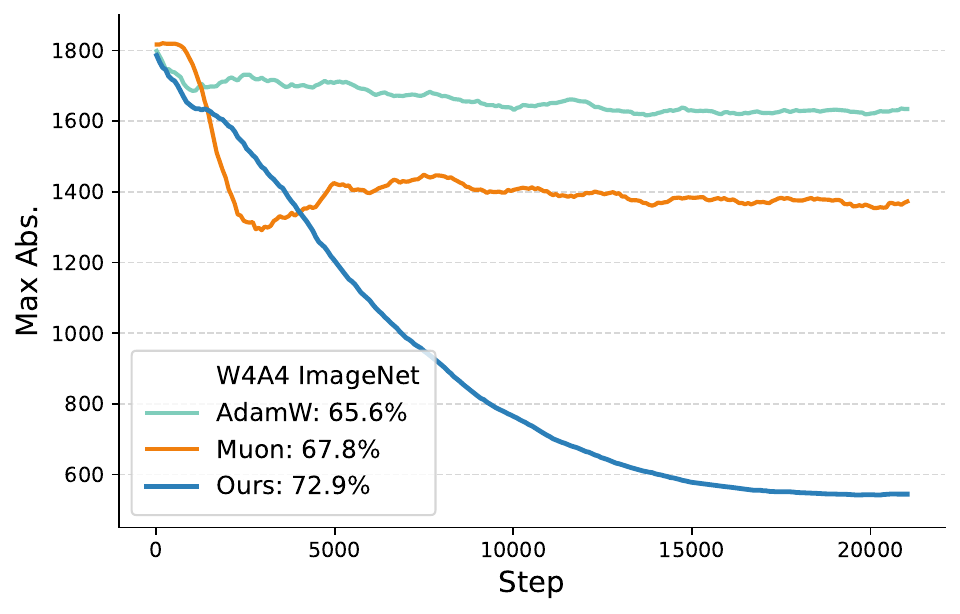}
    \caption{\textbf{Activation Outlier Suppression.}Comparison of the absolute maximum activation value (Max Abs.) of the Layer-9-FC1 output in the SigLIP-2 Base model. AdamW and Muon produce large activation outliers, whereas $S^2D$ substantially suppresses them, leading to improved downstream quantization performance as shown for W4A4.}
    \label{fig:s2d_activation_control}
\end{figure}

In this work, we empirically demonstrate that this problem escalates with the scale and duration of AdamW pre-training. Through a comparative analysis of the widely-used CLIP \cite{radford2021learningtransferablevisualmodels}, SigLIP \cite{zhai2023sigmoidlosslanguageimage}, and the more extensively trained SigLIP2 \cite{tschannen2025siglip2multilingualvisionlanguage}, we reveal a clear trend: the severity of activation outliers progressively increases with more extensive pre-training (see Figure \ref{fig:clip-outliers}). We posit that this phenomenon is a direct consequence of prolonged optimization with AdamW, whose core mechanism of adaptive, per-parameter gradient scaling is inherently anisotropic \cite{kingma2017adammethodstochasticoptimization}. Over millions of training iterations, these anisotropic updates introduce a \textit{privileged basis} in the model's representation space, where certain axes are preferentially amplified \cite{elhage2023privileged}, leading to the runaway magnitudes that characterize activation outliers. 

\begin{figure*}
    \centering
    \begin{subfigure}{0.32\textwidth}
        \includegraphics[width=\linewidth]{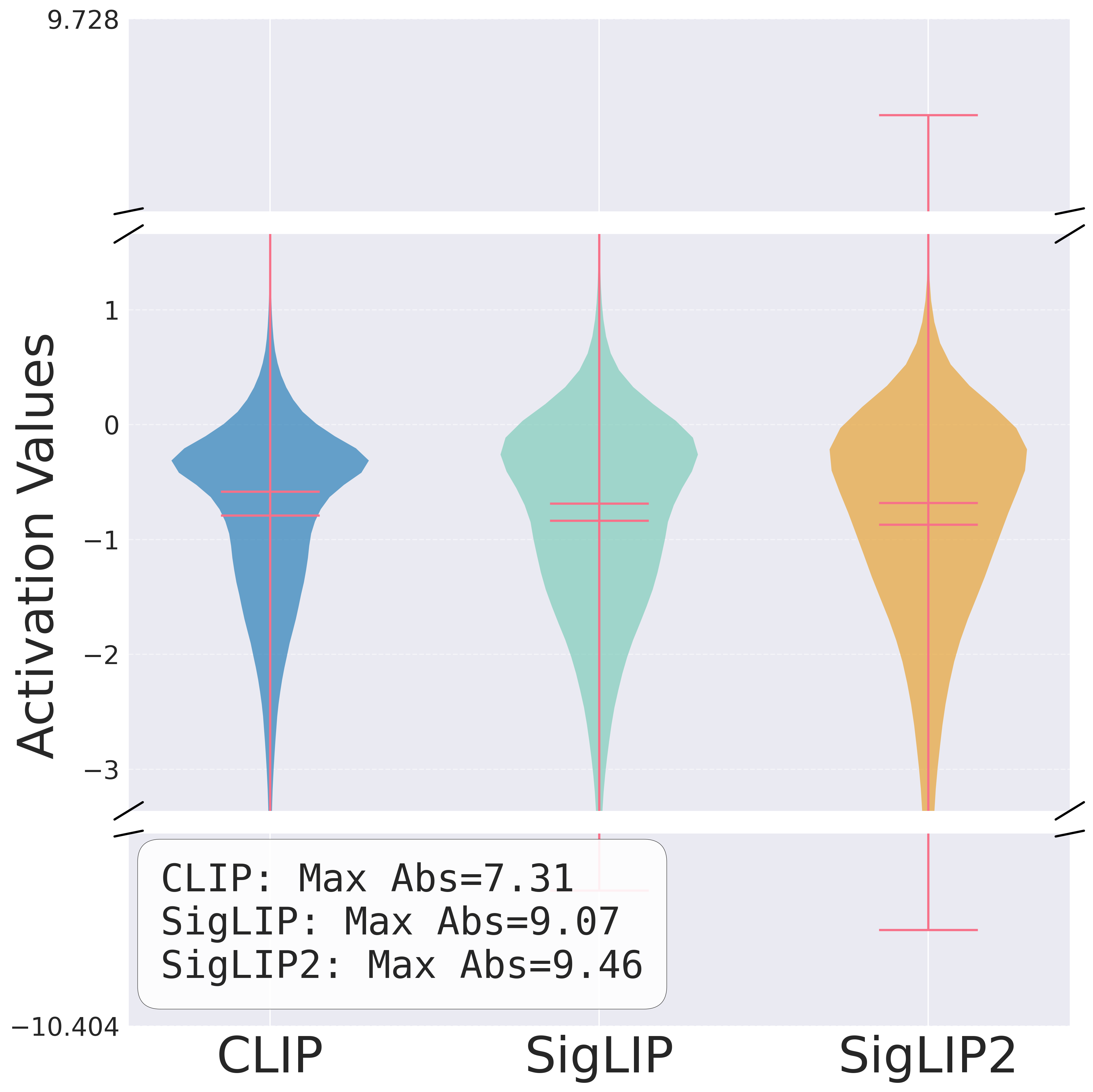}
        \caption{Layer 1}
        \label{fig:layer1_activations}
    \end{subfigure}
    \hfill % This command adds horizontal space between the figures
    \begin{subfigure}{0.32\textwidth}
        \includegraphics[width=\linewidth]{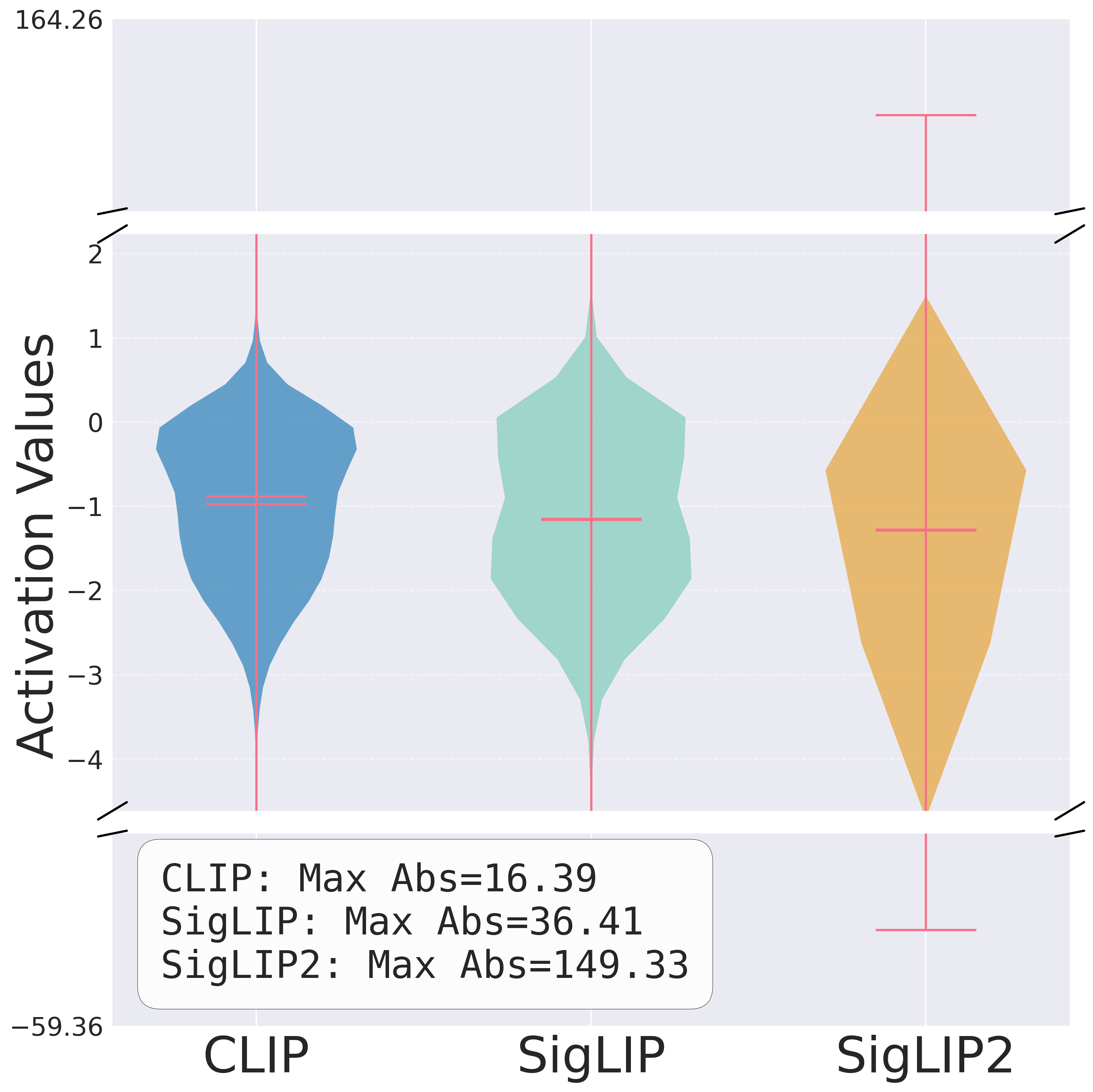}
        \caption{Layer 5}
        \label{fig:layer5_activations}
    \end{subfigure}
    \hfill % This command adds horizontal space between the figures
    \begin{subfigure}{0.32\textwidth}
        \includegraphics[width=\linewidth]{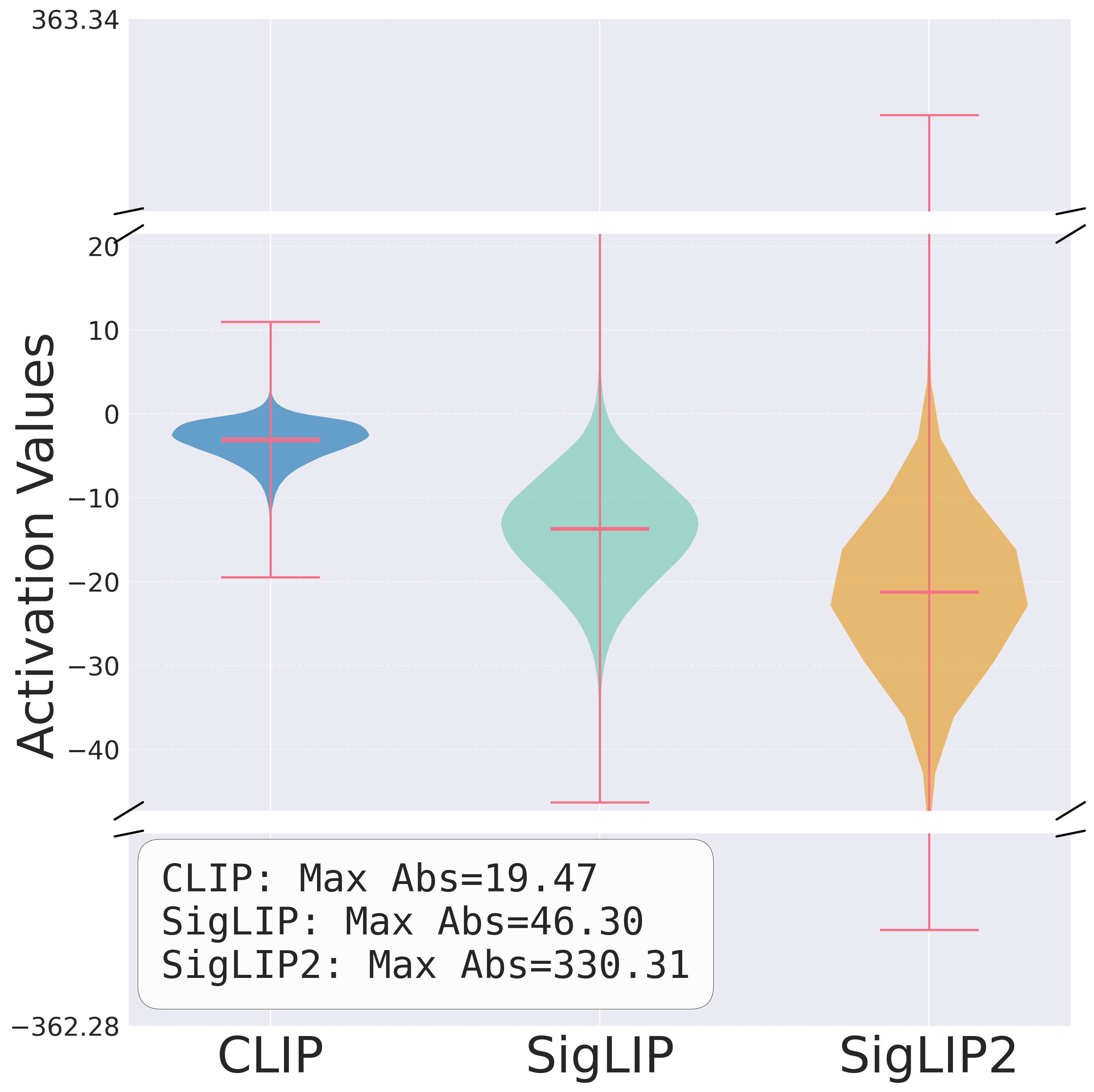}
        \caption{Layer 9}
        \label{fig:layer11_activations}
    \end{subfigure}
    \caption{
        \textbf{Activation outlier severity escalates with pre-training scale.} 
        Activation outlier severity escalates with increasing pre-training scale. 
The figure plots activation distributions from the feed-forward network (FFN) layers of ViT backbones across 
\textsc{CLIP}, \textsc{SigLIP}, and \textsc{SigLIP2}. 
A clear upward trend emerges: \textit{the magnitude of activation outliers consistently increase as we move from CLIP $\rightarrow$ SigLIP $\rightarrow$ SigLIP2}, 
highlighting the heavier-tailed activation behavior induced by larger and more recent model families.
    }
    \label{fig:clip-outliers}
\end{figure*}

The precise geometric mechanism that leads to these outliers is still not understood well. This paper establishes the direct link: the root cause of activation outliers is the uncontrolled growth of the spectral norm of the weight matrices (see Section \ref{sec:motivation}). A linear layer's capacity to amplify its input is fundamentally bounded by its spectral norm. We move beyond correlation and provide a diagnostic we term as \emph{Principal Component Dominance Ratio (PCDR)}. This metric quantifies what fraction of an activation's absolute magnitude comes from the top-$k$ singular components of the weight matrix. Our analysis reveals that activation outliers have a substantially higher PCDR$_k$, proving that these extreme values are generated by the inflated dominant singular components in the preceding weight matrix, while normal activations have significantly lower PCDR$_k$ values. 

To mitigate the occurrence of large outliers and stabilize training, orthogonal optimizers such as Muon \cite{jordan2024muon} have recently been proposed. However, these approaches are designed to train models from scratch, and when applied on an AdamW pre-trained model, the benefits are not too significant (see Figure \ref{fig:s2d_activation_control}). We propose Selective Spectral Decay ($S^2D$), a spectral conditioning method for correcting activation outliers in AdamW pre-trained models. One of the key advantages of $S^2D$ is that it works directly on existing pre-trained models without requiring expensive retraining from scratch. Figure \ref{fig:s2d_activation_control} shows that $S^2D$ is able to reduce outliers substantially compared to AdamW or Muon, improving downstream quantization performance. Using Singular Value Decomposition, $S^2D$ selectively regularizes only the largest singular values, the specific components causing outliers, while standard $L2$ weight decay uniformly shrinks all parameters. $S^2D$ can be applied during downstream fine-tuning or as a standalone post-processing step, producing well-conditioned models with maintained accuracy and improved robustness to quantization.

% The benefits of $S^2D$ extend beyond outlier mitigation. By controlling the spectral norm, $S^2D$ directly constrains the Lipschitz constant of each layer, promoting a smoother learned function that improves generalization and robustness against input perturbations \dpk{Need to demonstrate later through an experiment}. Furthermore, this spectral control offers a powerful mechanism to combat shortcut learning. Recent work has formalized that features associated with large singular values are more 'available' to a model, which can lead it to rely on spurious correlations (e.g., background texture) rather than core features (e.g., object shape) \cite{XXX}. By selectively suppressing these dominant singular values, $S^2D$ reduces the availability of such shortcuts, forcing the model to learn more robust representations. This allows for the targeted dismantling of the spectral artifacts inherited from pre-trained models (\emph{e.g. AdamW trained}) while preserving the valuable information stored in the smaller singular components. We show that $S^2D$ is most effective when paired with an isotropic optimizer like Muon, which avoids re-introducing the anisotropic biases that created the problem. 

\noindent\textbf{Our contributions are as follows.}
\begin{itemize}
\item We demonstrate that activation outlier severity escalates with pre-training scale and duration across vision-language models (\emph{e.g.}, CLIP $\rightarrow$ SigLIP $\rightarrow$ SigLIP2), establishing outliers as an inherent artifact of prolonged optimization with traditional optimizers such as AdamW.

\item We establish the direct link between inflated dominant singular values of weight matrices and activation outliers, and introduce the top-$k$ Principal Component Dominance Ratio (PCDR$_k$) as a diagnostic metric.

\item We propose Selective Spectral Decay ($S^2D$), a geometrically-principled regularizer that selectively penalizes largest singular values during fine-tuning, suppressing the spectral pathologies responsible for outliers while preserving useful model capacity.

\item We demonstrate through extensive experiments that $S^2D$ produces well-conditioned, quantization-ready models and push the performance of existing state-of-the-art quantization methods.

\end{itemize}
% \ar{Overall about the intro: i think the claims and results should be more concrete: what datasets you use, what gains you see etc. We could also separate out intuitions vs formal claims vs actual experimental contributions a little better. It's also unclear whether we are focusing on quantization, or finetuning (or both?).}
\section{Related Works}

The phenomenon of activation outliers, reflected through extreme values that appear consistently in specific feature dimensions, had emerged as a critical challenge in deploying large-scale neural networks. \citet{dettmers2022llmint88bitmatrixmultiplication} characterized this phenomenon in LLMs, demonstrating that outlier features can exhibit magnitudes up to 150,000 times larger than typical activations. \citet{xiao2024smoothquantaccurateefficientposttraining} showed across multiple transformer architectures that outlier dimensions are highly consistent across tokens and that outlier severity increases in deeper layers. \citet{yao2022zeroquantefficientaffordableposttraining} extended this analysis to show that outliers appear across different model families and scales, with severity generally increasing with model size. \citet{wei2023outliersuppressionaccuratequantization} further demonstrated that outlier patterns persist across different training runs and are reproducible, suggesting they arise from fundamental properties of the training process rather than random initialization effects.

While the majority of outlier research has focused on language models, a few recent works have reported it in the context of vision transformers and multimodal models \cite{darcet2024visiontransformersneedregisters}. Our work contributes to this line of research by providing the first comparative analysis varying pre-training durations (CLIP, SigLIP, SigLIP2) and demonstrating that outlier severity correlates with training duration rather than model capability or task complexity. This observation provides some evidence that outliers are optimization artifacts rather than functionally necessary features.

The impact of activation outliers on quantization has been extensively studied, as they posed a fundamental challenge to model compression. \citet{dettmers2022llmint88bitmatrixmultiplication} demonstrated that even a single outlier dimension can catastrophically affect the process of uniform quantization. Extreme outlier values force the scale factor to be very large, causing the vast majority of normal-magnitude activations to be rounded to zero or very small integers, thereby leading to \textit{quantization collapse}. They proposed to process outlier dimensions (approximately 0.1\% of features) in FP16, and the remaining in INT8. This vector-wise quantization preserves accuracy but requires dynamic routing and specialized kernels. \cite{liu2024qllmaccurateefficientlowbitwidth} presented QLLM that extends this with more sophisticated outlier detection and handling mechanisms. An alternative line of work attempts to reduce outlier impact through mathematically equivalent transformations \cite{shao2024omniquantomnidirectionallycalibratedquantization, liu2025spinquantllmquantizationlearned, tan2024mobilequantmobilefriendlyquantizationondevice}. SmoothQuant \cite{xiao2024smoothquantaccurateefficientposttraining} mitigated the quantization difficulty by migrating outlier issue from activations to weights through per-channel scaling. Outlier Suppression \cite{wei2023outliersuppressionaccuratequantization} proposes channel-wise shifting and scaling operations that equivalently transform the network to reduce outlier magnitudes.

% RepQViT PTQ4ViT and ERQ more than other ViT PTQ methods
Concurrently, a complementary line of work has developed PTQ methods tailored to the unique architectural challenges of Vision Transformers. RepQViT \cite{li2023repqvitscalereparameterizationposttraining} introduced specialized handling for post-LayerNorm (using channel-wise quantization) and post-Softmax (using $\log\sqrt{2}$) quantization, later transforming them to simple quantizers through scale reparameterization. PTQ4ViT \cite{yuan2024ptq4vitposttrainingquantizationvision} addresses the \textit{unbalanced} post-Softmax and \textit{asymmetric} post-GELU outputs, by proposing a \textit{twin uniform quantization} scheme that uses separate, hardware-friendly quantizers for different value ranges. ERQ \cite{zhong2025accurateposttrainingquantizationvision} introduces an innovative two-step framework to sequentially reduce both activation and weight quantization errors. By formulating their minimization as a Ridge Regression problem, ERQ addresses the intricate interdependence between weight and activation errors, thereby significantly outperforming existing ViT and LLM-centric approaches like SpinQuant \cite{liu2025spinquantllmquantizationlearned} and OmniQuant \cite{shao2024omniquantomnidirectionallycalibratedquantization} on ViT models.

Our work differs from existing approaches in several key aspects. Unlike methods that work around outliers (mixed-precision, smoothing), we address their root cause by conditioning the weight matrices to reduce spectral imbalance. Unlike QAT approaches that require  retraining, $S^2D$ can be applied to existing pre-trained models. Importantly, $S^2D$ is complementary to existing quantization methods: by producing well-conditioned models with reduced outliers, $S^2D$ creates better starting points for PTQ methods, and can be naturally combined with QAT during task-specific fine-tuning to achieve even better quantization robustness.
\section{Motivation}
\label{sec:motivation}
A well-known barrier to effective quantization is the presence of activation outliers, which force most normal activations to be squished into a narrow dynamic range, leading to sub-optimal bin allocation and ultimately degrading model accuracy. Existing works have identified this as a critical problem, but the root causes remain poorly understood. Our goal is to understand: \textit{where do these outliers come from, and can we eliminate them without retraining from scratch?} Answering these questions requires investigating the relationship between pre-training dynamics and outlier formation. We provide two key empirical observations that motivate the development of $S^2D$. First, we observed that the problem of activation outliers is not a static issue but rather one that escalates with the scale and duration of pre-training in foundational vision models. Second, moving beyond correlation we established a direct link, showing that these outliers are generated by the dominant singular components of the weight matrices they pass through.

\noindent{\textit{\textbf{Outliers scale with pre-training.} }}We demonstrate here that the problem of outliers intensifies systematically with pre-training scale and duration. To investigate this relationship, we analyze three foundational vision encoder models: CLIP, SigLIP, and SigLIP2. This progressions represent substantial increases in training compute and data scale, providing controlled observations of long-term AdamW pre-training effects. Figure \ref{fig:clip-outliers} shows the progressive emergence of activation outliers across three representative layers of the CLIP family. Compared to CLIP, SigLIP shows significantly increased outlier severity represented through wider tails, and this phenomenon exponentially amplifies further for SigLIP2. Note that this degradation can be seen across layers and this is evident from the distributions shown for Layers 1, 5 and 9.

\begin{table}[t!]
    \centering
    \caption{
        \textbf{Principal Component Dominance Ratios for Outlier Activations.}
        Top-1 and Top-3 PCDR for FFN layer activations, along with the maximum singular value of the corresponding FFN weights across CLIP, SigLIP and SigLIP2. $\sigma_{max}$ denotes largest singular value.
    }
    \label{tab:pcd_ratios_comparison} % Updated label
    % Updated column format to: Model, Layer, Top-3, Top-5
    \resizebox{0.9\linewidth}{!}{
    \begin{tabular}{llccc}
        \toprule
        \textbf{Model} & \textbf{Layer} & \textbf{PCDR$_{1}$} & \textbf{PCDR$_{3}$} & \textbf{$\sigma_{max}$} \\
        \midrule
        
        % Using \multirow to group layers under each model
        \multirow{3}{*}{CLIP} 
         & Layer 1 & 0.08 & 0.09 & 3.78 \\
         & Layer 5 & 0.05 & 0.14 & 3.06 \\
         & Layer 9 & 0.18 & 0.19 & 4.81 \\
         % & Layer 10 & XXX & XXX & \\
        \midrule
        
        \multirow{3}{*}{SigLIP} 
         & Layer 1 & 0.01 & 0.02 & 3.40 \\
         & Layer 5 & 0.03 & 0.89 & 4.86\\
         & Layer 9 & 0.14 & 0.14 & 7.59 \\
         % & Layer 10 & XXX & XXX \\
        \midrule
        
        % Original SigLIP2 data, with only Top-3 and Top-5 kept
        \multirow{3}{*}{SigLIP2} 
         & Layer 1 & 0.02 & 0.06 & 4.53 \\
         & Layer 5 & 0.93 & 0.95 & 10.8 \\
         & Layer 9 & 0.88 & 0.99 & 8.19 \\
         % & Layer 10 & 0.865 & 0.865 \\
        \bottomrule
    \end{tabular}
    }
\end{table}

\noindent{\textit{\textbf{Spectral decomposition of activation outliers}}. }The correlation between training scale and outlier severity raises an important question: which components of the weight matrices are responsible for generating extreme activation values? We hypothesize that outliers are not distributed across all spectral components but are disproportionately generated by the dominant singular values or spectral norm of weights. To test this, we perform a spectral decomposition analysis using what we term as the \textit{Principal Component Dominance Ratio (PCDR$_k$)}, which quantifies how much of an individual activation's magnitude originates from the largest few (top-$k$) singular components.

For the $j^{\text{th}}$ data sample, given an input activation vector $\mathbf{x}_j$ to a layer with weight matrix $\mathbf{W} = \mathbf{U\Sigma V}^{\intercal}$, the output activation for neuron $i$ can be decomposed as: 
\begin{equation} 
    A_{ij} = \sum_r \sigma_r u_{ir} \mathbf{v}_{r}^{\intercal} \mathbf{x}_j
\end{equation} 
where $r$ denotes the total number of singular directions.

To account for the activation mass contributed by the top-$k$ singular values, we define PCDR$_k$ as the fraction of the activation's magnitude that comes from the top $k$ singular components.
\begin{equation} 
\text{PCDR}_{k}^{(i,j)} = \frac{\sum_{r=1}^{k} \big|\sigma_r u_{ir} \mathbf{v}_r^{\intercal} \mathbf{x}_j \big|}{\sum_r \big|\sigma_r u_{ir} \mathbf{v}_{r}^{\intercal} \mathbf{x}_j \big|}, 
\end{equation}

where a PCDR$_k$ value of close to 1 indicates that the activation value is almost entirely determined by the top-$k$ components, while values near $1/n$ (where $n$ is the total number of components) indicate contributions are uniformly distributed.

We compute PCDR$_{k}$ for the largest activations value $A_{ij}$ in the FFN layer of ViT models. Results are presented in Table \ref{tab:pcd_ratios_comparison}. It can be seen that PCDR$_3$ increases and approaches values closer to 1 as we scale from CLIP to SigLIP to SigLIP2. This analysis establishes that outliers are not uniformly generated by the entire weight matrix but are specifically produced by inflated dominant singular values.

% \begin{figure*}
%     \centering
%     \begin{subfigure}{0.32\textwidth}
%         \includegraphics[width=\linewidth]{assets/violin_comparison_vision_model_encoder_layers_1_mlp_fc1.png}
%         \caption{Layer 1}
%         \label{fig:layer1_activations}
%     \end{subfigure}
%     \hfill % This command adds horizontal space between the figures
%     \begin{subfigure}{0.32\textwidth}
%         \includegraphics[width=\linewidth]{assets/violin_comparison_vision_model_encoder_layers_5_mlp_fc1.png}
%         \caption{Layer 5}
%         \label{fig:layer5_activations}
%     \end{subfigure}
%     \hfill % This command adds horizontal space between the figures
%     \begin{subfigure}{0.32\textwidth}
%         \includegraphics[width=\linewidth]{assets/violin_comparison_vision_model_encoder_layers_10_mlp_fc1.png}
%         \caption{Layer 10}
%         \label{fig:layer11_activations}
%     \end{subfigure}
%     \caption{\dpk{[PLACEHOLDER FOR THE QWEN FAMILY PLOTS]}
%         \textbf{Activation outlier severity escalates with pre-training scale and model depth.} 
%         Violin plots of activation distributions from the feed-forward network (FFN) layers of ViT backbones for CLIP, SigLIP, and SigLIP2.
%     }
%     \label{fig:qwen-outliers}
% \end{figure*}

\section{Mathematical Formulation}

In this section, we first establish the link between the spectral properties of a layer's weight matrix and the magnitude of its output activations. This provides a formal basis for our central hypothesis that activation outliers are a direct consequence of an inflated spectral norm. Building on this, we formulate our proposed regularizer, Selective Spectral Decay ($S^2D$) and detail its mechanism along with an efficient implementation.

\noindent{\textbf{\textit{Preliminaries.} }} The fundamental building block of a neural network is the linear layer, which performs the transformation $\mathbf{y} = \mathbf{Wx}$ for a weight matrix $\mathbf{W} \in \mathbb{R}^{m \times n}$. The geometric properties of this transformation are characterized by the Singular Value Decomposition (SVD) of $\mathbf{W}$:
\begin{equation}
    \mathbf{W} = \mathbf{U \Sigma V}^{\intercal} = \sum_{r=1}^{N} \sigma_r \mathbf{u}_r \mathbf{v}_r^{\intercal}
    \label{eq:svd}
\end{equation}
where $N = \min(m, n)$, the columns of $\mathbf{U} \in \mathbb{R}^{m \times m}$ and $\mathbf{V} \in \mathbb{R}^{n \times n}$ are the orthonormal left and right singular vectors, respectively, and $\Sigma \in \mathbb{R}^{m \times n}$ is a rectangular diagonal matrix containing the singular values $\sigma_1 \ge \sigma_2 \ge \dots \ge \sigma_N \ge 0$. 
%Geometrically, this decomposition describes any linear map as a rotation ($V^T$), followed by a coordinate-wise scaling ($\Sigma$), followed by another rotation ($U$).

The most common form of regularization, L2 weight decay, penalizes the squared Frobenius norm of the weight matrix:
\begin{equation}
    \mathcal{L}_{2} = \frac{\lambda}{2} \|\mathbf{W}\|_F^2 = \frac{\lambda}{2} \sum_{i=1}^{N} \sigma_i^2
    \label{eq:l2_decay}
\end{equation}
This penalty applies a uniform decay pressure to all singular values, regardless of their magnitude. While effective for general-purpose regularization, it is not specifically designed to target the spectral artifacts that we have shown, are responsible for activation outliers. 
% \ar{I think the experiments do not show that Frobenius norm does *not* go up as we see more outliers?}\dpk{What is the action on this comment?}

\noindent{\textbf{\textit{The Spectral Origin of Activation Outliers.} }} We now formalize the link between the spectral norm of a weight matrix and its capacity to generate large-magnitude activations.
% \ar{I don't think this has to be a theorem... it's just the definition}
\begin{theorem}
\label{thm:activation_bound}
Let $\mathbf{y} = \mathbf{Wx}$ be the output of a linear layer for an input vector $\mathbf{x} \in \mathbb{R}^n$ and weight matrix $\mathbf{W} \in \mathbb{R}^{m \times n}$. The Euclidean norm of the output vector is bounded by the spectral norm of the weight matrix $\sigma_{\max}(\mathbf{W})$, as follows:
\begin{equation}
    \|\mathbf{y}\|_2 \le \sigma_{\max}(\mathbf{W}) \cdot \|\mathbf{x}\|_2
\label{prop:1}
\end{equation}
% \textbf{Proof.} See Supplementary Material.
\end{theorem}

% \begin{proof}
% Let \(\|\cdot\|_2\) denote the Euclidean norm on vectors.  
% The matrix norm induced by \(\|\cdot\|_2\) is defined as
% \[
%     \|\mathbf{W}\|_2 := \sup_{\mathbf{x} \neq \mathbf{0}} \frac{\|\mathbf{Wx}\|_2}{\|\mathbf{x}\|_2} = \sigma_{\max}(\mathbf{W})
% \]
% Therefore, for any \(\mathbf{x} \in \mathbb{R}^n\),
% \[
%     \|\mathbf{y}\|_2 = \|\mathbf{Wx}\|_2 \le \|\mathbf{W}\|_2 \|\mathbf{x}\|_2 = \sigma_{\max}(\mathbf{W})\|\mathbf{x}\|_2.
% \]
% \end{proof}
This establishes that a large spectral norm is a necessary condition for a layer to produce a large-magnitude output from a reasonably scaled input, providing a direct mechanism for the amplification of activation magnitude. 

\subsection{Selective Spectral Decay ($S^2D$)}

Having established that the inflation of the largest singular values is the primary mechanism behind activation outliers, we introduce a regularizer that specifically targets this specific behavior.

\begin{definition}
Given a weight matrix $\mathbf{W} = \mathbf{U \Sigma V}^\intercal$, we define $\mathcal{W}^{(n)} = \mathbf{U \Sigma}^n \mathbf{V}^\intercal$ for a real exponent $n > 1$. The Selective Spectral Decay regularizer is then defined as
\begin{equation}
\mathcal{L}_{\textnormal{S}^\textnormal{2}\textnormal{D}}^{(n)}(\mathbf{W}) = \frac{\lambda}{n+1} \operatorname{tr}\big((\mathcal{W}^{(n)})^\intercal\mathbf{W}\big) \end{equation} 
\textnormal{where $\operatorname{tr}$ denotes the trace. By orthogonality of $\mathbf{V}$ and $\mathbf{U}$}
\[\operatorname{tr}\!\big((\mathcal{W}^{(n)})^\intercal\mathbf{W}\big) 
= \operatorname{tr}(\mathbf{V \Sigma}^{n+1} \mathbf{V}^\intercal),
\]
\textnormal{and by cyclicity of trace,}
\[
\operatorname{tr}(\mathbf{V \Sigma}^{n+1} \mathbf{V}^\intercal) 
= \operatorname{tr}(\mathbf{\Sigma}^{n+1}) 
= \sum_{i=1}^N \sigma_i^{n+1}.
\]
\textnormal{Thus, we obtain}
\[
\mathcal{L}_{\textnormal{S}^\textnormal{2}\textnormal{D}}^{(n)}(\mathbf{W}) 
= \frac{\lambda}{n+1} \sum_{i=1}^N \sigma_i^{n+1}.
\]
\end{definition}

By choosing $n > 1$, the penalty $\sigma_i^{n+1}$ disproportionately affects larger singular values while having a little to negligible effect on smaller ones. This provides a directed mechanism for suppressing the spectral inflation compared to standard L2 decay (which corresponds to $n=1$). 
This allows us to penalize those \(W_{ij}\) proportionately which are not just large in value, but specifically large due to the influence of a larger singular components in the system.

Based on the above formulation, the standard partial gradients of the \(L2\) regularizer can be modified from:
\begin{flalign*}
& \frac{\partial}{\partial W_{ij}}\!\left(\frac{\lambda}{2}\|\mathbf{W}\|_F^2\right)
= \lambda W_{ij}
= \lambda \sum_{k=1}^{N} U_{ik}\,\sigma_k\,V_{jk}. & \\
\intertext{and rewritten for $S^2D$ as:}
& \frac{\partial}{\partial W_{ij}}\!\left(\frac{\lambda}{n+1} \operatorname{tr}\big((\mathcal{W}^{(n)})^\intercal\mathbf{W}\big) \right)
= \lambda W^{(n)}_{ij} \\
&= \lambda \sum_{k=1}^{N} U_{ik}\,\sigma_k^n\,V_{jk}. &
\end{flalign*}

This formulation focuses the regularization pressure on the top few $\sigma_i$ - the singular values directly responsible for the worst-case amplification of activations as shown in Theorem \ref{thm:activation_bound}.

\subsection{$S^2D$ in Action}
\label{sec:implementation}

$S^2D$ regularizer provides a powerful tool for penalizing dominant singular values. However, a naive implementation that computes a full SVD and applies the gradient to all singular values of all layers at every training step is computationally prohibitive and unnecessary. As our analysis in Section \ref{sec:motivation} demonstrated, activation outliers are a pathological phenomenon driven by a few dominant components in a subset of layers.

Therefore, a practical and efficient implementation of $S^2D$ must be both selective and computationally amortized. We achieve this through a PCDR$_{k}$-based criterion and a staggered update schedule.

\noindent\textbf{\textit{PCDR$_{k}$-based Selection.} }
We directly use PCDR$_{k}$ to identify which layers require regularization and, for those layers, which singular components to target. This selection process is governed by two hyperparameters: (1) $\tau$ -- The minimum PCDR contribution that signifies a pathological concentration of mass. (2) $K_{max}$ -- The maximum number of dominant singular values to consider.

For a given layer, we find the smallest rank $k_{target}$ such that $1 \le k_{target} \le K_{max}$ and PCDR$_{k_{target}}$ $\ge \tau$. If such a $k_{target}$ exists, the layer is marked for regularization, and the $S^2D$ penalty is applied only to its top $k_{target}$ singular components. If PCDR$_{K_{max}}$ $< \tau$, the layer is considered healthy, and no $S^2D$ gradient is applied.

\noindent\textbf{\textit{Amortized SVD Computation.} }
The primary computational bottleneck of $S^2D$ is the SVD computation itself. Instead of re-computing the SVD at every step, we perform a full SVD on all network layers only once every $m$ iterations. This step identifies the target layers and their corresponding $k_{target}$ ranks, and caches the singular vectors ($U$, $V$) and singular values ($\Sigma$) for those layers.
For the subsequent $m$ iterations, we apply the $S^2D$ gradient using these \textit{stale} cached components. While this introduces a minor approximation (as the weight matrix $\mathbf{W}$ evolves during these steps), it amortizes the high cost of SVD over $m$ steps, making the algorithm highly efficient. 

\begin{table*}[ht]
\centering
\caption{\textbf{SigLIP2 quantization performance on ImageNet1k}. Comparing ImageNet1k performance on AdamW and AdamW+$S^2D$ (Ours) across various weight (W) / activation (A) precision settings and post-training quantization (PTQ) methods. The table demonstrates that $S^2D$ shows consistent improvements for W4A4, W5A5, W6A6, and W8A8 configurations, when using ERQ, PTQ4ViT, and RepQ-ViT to perform PTQ.}
\label{tab:transposed_ptq}
\resizebox{\linewidth}{!}{
\begin{tabular}{llccccccccccccccc} 
\toprule

\multirow{2}{*}{\textbf{Res.}} &
\multirow{2}{*}{\textbf{Method}} &
\multirow{2}{*}{\textbf{FP16}} &
\multicolumn{4}{c}{\textbf{RepQ-ViT}} &
\multicolumn{4}{c}{\textbf{PTQ4ViT}} &
\multicolumn{4}{c}{\textbf{ERQ}} \\
\cmidrule(lr){4-7}
\cmidrule(lr){8-11}
\cmidrule(lr){12-15}
 & & & \textbf{W4A4} & \textbf{W5A5} & \textbf{W6A6} & \textbf{W8A8} 
         & \textbf{W4A4} & \textbf{W5A5} & \textbf{W6A6} & \textbf{W8A8}
         & \textbf{W4A4} & \textbf{W5A5} & \textbf{W6A6} & \textbf{W8A8} \\
\midrule

% ---------------- Base-384 ----------------
\multirow{2}{*}{\textbf{384}}
 & \textbf{AdamW} 
 & 85.0 
 & 37.4 & 46.0 & 58.5 & 83.0
 & 3.4  & 44.4 & 65.3 & 83.4
 & 65.6 & 78.5 & 81.1 & 83.8 \\

 & \textbf{Ours}
 & 85.0
 & \textbf{41.4} & \textbf{78.0} & \textbf{80.0} & \textbf{83.5}
 & \textbf{3.9}  & \textbf{62.0} & \textbf{81.5} & \textbf{84.1}
 & \textbf{73.0} & \textbf{81.9} & \textbf{82.7} & \textbf{84.0} \\
\midrule

% ---------------- Base-512 ----------------
\multirow{2}{*}{\textbf{512}}
 & \textbf{AdamW}
 & 85.3
 & 16.3 & 43.3 & 53.9 & 83.3
 & \textbf{4.2} & 46.9 & 80.1 & 84.5
 & 58.6 & 77.6 & 80.1 & 78.6 \\

 & \textbf{Ours}
 & 85.4
 & \textbf{17.2} & \textbf{79.2} & \textbf{82.7} & \textbf{84.3}
 & 3.8           & \textbf{77.0} & \textbf{83.0} & \textbf{84.9}
 & \textbf{68.1} & \textbf{82.3} & \textbf{83.5} & \textbf{84.1} \\

\bottomrule
\end{tabular}
}
\end{table*}

\section{Experiments}

Our experimental evaluation focuses on validating the effectiveness of $S^2D$ in producing quantization-friendly models across diverse settings. The experiments presented in this paper are designed around three important questions: (1) Does $S^2D$ effectively reduce activation outliers that hinder quantization? (2) Does $S^2D$ improve quantization performance across both PTQ and QAT regimes? (3) Do these quantization improvements generalize to downstream vision tasks?

\textbf{Hyperparameters for $S^2D$.} We use the following hyperparameters across all experiments: $\tau = 0.95$, $K_{\max}=3$, $m=100$, $n=2$, and $\lambda = 5\times 10^{-4}$. Here, $\tau$ is the PCDR threshold indicating concentrated spectral mass, $K_{\max}$ is the maximum number of dominant singular values considered, $m$ controls the interval (in steps) between $S^2D$ updates, $n$ is the power used in $S^2D$, and $\lambda$ is the decay strength. Additional algorithmic details, sensitivity analyses, and full training settings are provided in the Supplementary Material.

\subsection{Outlier Severity Across Model Scale}
We intiated our analysis by establishing the empirical foundation: activation outliers intensify with pre-training scale, motivating the need for conditioning methods like $S^2D$. Building on the motivational analysis presented in Section \ref{sec:motivation}, we quantify outlier severity across CLIP, SigLIP, and SigLIP2 vision models using metrics including maximum activation magnitude and the PCDR$_k$.

Figures \ref{fig:clip-outliers} demonstrate a clear monotonic trend that outlier severity systematically increases from CLIP to SigLIP to SigLIP2, showing a correlation with training duration and compute. All three models use the exact same ViT-Base \cite{wu2020visual} architecture which establishes outliers as a consequence of prolonged pre-training rather than architecture-specific artifacts. This observation motivates our focus on SigLIP2 for the various experiments in the paper. As outliers are most prominent in heavily pre-trained models, the quantization challenge are greatest in this regime.

\subsection{Post-Training Quantization (PTQ)}
\label{sec-ptq}

To evaluate $S^2D$'s capability in producing quantization-friendly models, we conduct PTQ experiments on ImageNet-1k classification. We initialize from the pre-trained SigLIP2 backbone and fine-tune using AdamW optimizer as the baseline method and AdamW+$S^2D$. Using both the approaches, the pre-trained model is fine-tuned for 10 epochs with similar hyperparamters and augmentations as outlined in \cite{wortsman2022model} to produce full-precision checkpoints. These fine-tuned models serve as inputs for post-training quantization.
For PTQ, we use the current state-of-the-art PTQ method, ERQ \cite{zhong2025accurateposttrainingquantizationvision}; a SOTA vision transformer PTQ method, PTQ4ViT \cite{yuan2024ptq4vitposttrainingquantizationvision}; and a re-parameterization based quantization method RepQ-ViT \cite{li2023repqvitscalereparameterizationposttraining} and quantize the full-precision models to different bit sizes. ERQ is competitive with all exisiting PTQ methods across vision and language transformers and hence serves as a strong PTQ baseline. Evaluation results for this experimental setup are reported in Table \ref{tab:transposed_ptq}.

% Based on Table \ref{tab:transposed_ptq}, it is clear that $S^2D$–fine-tuned models substantially outperform standard AdamW across all PTQ methods and bit-settings. For SigLIP2-Base-384 under the \textsc{ERQ} W4A4 setting, $S^2D$ achieves 72.99\% compared to 65.58\% for AdamW, a significant 7.41\% percentage point improvement. A larger gain of 17.52\% and 16.18\% is observed with PTQ4ViT; W5A5 and W6A6 setting respectively. The trend persists under \textsc{RepQ-ViT}, where $S^2D$ increases W5A5 accuracy from 46.04\% to 78.07\% and W6A6 accuracy from 58.49\% to 79.98\%. This consistency across PTQ strategies, strongly suggests that the benefits of $S^2D$ stem from improved weight conditioning rather than method-specific interactions. Importantly, full-precision performance is essentially preserved, confirming that $S^2D$ conditions the weight geometry specifically for quantization robustness without sacrificing the model's inherent representational capacity. This indicates that $S^2D$'s spectral regularization targets only the pathological components introduced by prolonged AdamW pre-training—the inflated dominant singular values—while preserving the useful information encoded in the weight matrices. This allows $S^2D$ to function as a pure conditioning method that produces better-conditioned models ready for deployment across both full-precision and quantized settings.

Based on Table \ref{tab:transposed_ptq}, it is clear that $S^2D$–fine-tuned models substantially outperform standard AdamW across all PTQ methods and bit-settings. For SigLIP2-Base-384 under the \textsc{ERQ} W4A4 configuration, $S^2D$ achieves 72.99\% versus 65.58\% for AdamW, a sizable 7.41-point improvement. Even larger gains of 17.52 and 16.18 points are observed with \textsc{PTQ4ViT} under the W5A5 and W6A6 settings, respectively. The trend extends to \textsc{RepQ-ViT}, where $S^2D$ improves W5A5 accuracy from 46.04\% to 78.07\% and W6A6 from 58.49\% to 79.98\%. This consistent improvement across diverse PTQ strategies strongly suggests that the benefits of $S^2D$ arise from fundamentally better weight conditioning rather than method-specific interactions. Importantly, full-precision accuracy is essentially preserved, confirming that $S^2D$ reshapes the weight geometry specifically for quantization robustness without diminishing the model’s inherent representational capacity. Overall, these results indicate that $S^2D$’s spectral regularization selectively suppresses the pathological components introduced by prolonged AdamW pre-training while preserving the useful information encoded in the weight matrices. This allows $S^2D$ to function as a pure conditioning method that produces well-conditioned models ready for deployment across both full-precision and quantized settings.

% \begin{table*}[t!]
% \caption{
%  % Updated and more descriptive caption
% \textbf{Comparison of FFN Activation and Weight Statistics for Siglip2 and DinoV3.}
% We report the $PCDR_1$ and maximum absolute activation of the FFN layers, along with the maximum singular value of their corresponding weights, after fine-tuning with AdamW and AdamW+$S^2D$.}
% \centering
%  % Increased resizebox width to accommodate the new column
% \resizebox{0.8\linewidth}{!}{
% \begin{tabular}{llcccccc}
% \toprule
%  % Added "Model" column header
% & & \multicolumn{2}{c}{\textbf{PCDR_{1}}} & \multicolumn{2}{c}{\textbf{Max. Abs. Activation}} & \multicolumn{2}{c}{\textbf{Max Singular Value}} \\
% \cmidrule(lr){3-4} \cmidrule(lr){5-6} \cmidrule(lr){7-8}
% \textbf{Model} & \textbf{Layer} & \textbf{AdamW} & \textbf{Ours} & \textbf{AdamW} & \textbf{Ours} & \textbf{AdamW} & \textbf{Ours} \\
% \midrule
%  % Used \multirow for "Siglip2" to span 2 rows
% \multirow{2}{*}{Siglip2} & Layer 5 & 0.91 & 0.46 & 176.40 & 59.68 & 10.38 & 3.22 \\
% & Layer 9 & 0.77 & 0.09 & 1166.2 & 614.7 & 7.87 & 3.96 \\
% \midrule % Added a rule to separate the models
%  % Used \multirow for "DinoV3" to span 2 rows
% \multirow{2}{*}{DinoV3} & Layer 2 & 0.47 & 0.11 & 1048.41 & 440.38 & 37.29 & 17.94 \\
% & Layer 4 & 0.45 & 0.22 & 115.55 & 67.56 & 8.67 & 3.79 \\
% \bottomrule
% \end{tabular}
% }
% \label{tab:pcd_and_singular_values_comparison} % Label remains the same
% \end{table*}
\begin{table}[t!]
\hfill
\begin{minipage}{0.46\textwidth}
\caption{
\textbf{Improved Layer Conditioning.} PCDR\(_1\), maximum absolute activation, and maximum singular value for the FC1 layers of SigLIP2 after fine-tuning with AdamW and AdamW+$S^2D$ (Ours). 
The results show that $S^2D$ consistently reduces \textsc{PCDR\(_1\)} compared to AdamW, indicating better spectral concentration. 
We also observe notable decreases in absolute max activation (Max Abs.) and leading singular values.
}
\label{tab:pcd_after}
\centering
\begin{tabular}{lccc}
\toprule
\textbf{Metric} & \textbf{Optimizer} & \textbf{Layer 5} & \textbf{Layer 9} \\
\midrule
\multirow{2}{*}{PCDR\(_1\)} & AdamW & 0.91 & 0.77 \\
                            & Ours & 0.46 & 0.09 \\
\midrule
\multirow{2}{*}{Max Abs.}    & AdamW & 176.4 & 1166.2 \\
                             & Ours & 59.7 & 614.7 \\
\midrule
\multirow{2}{*}{$\sigma_{\max}$} & AdamW & 10.4 & 7.9 \\
                                 & Ours & 3.2 & 3.9 \\
\bottomrule
\end{tabular}
\end{minipage}
\end{table}

% \begin{figure}
% \centering
% \includegraphics[width=0.9\columnwidth]{assets/placeholder.png}
% \caption{\dpk{PLACEHOLDER IMAGE}Violin plots of activation distributions from the feed-forward network (FFN) layers of Siglip2 fine-tuned using AdamW and AdamW+$S^2D$.}
% \label{fig:ptq-layers}
% \end{figure}
To validate the link between spectral concentration and quantization performance, we analyze the per-layer distribution of PCDR$_1$ metrics in models fine-tuned with and without $S^2D$. Table \ref{tab:pcd_after} demonstrates the aggregate improvements in spectral concentration for models trained using AdamW and AdamW+$S^2D$, respectively. The AdamW baseline exhibits a wide range of activation magnitudes, whereas $S^2D$ significantly reduces large-magnitude activations without negatively affecting performance. This effect is crucial for quantization, as a single poorly conditioned layer can dominate the activation range and force suboptimal scaling decisions. By explicitly regularizing large dominant singular values, $S^2D$ reduces the maximum singular value in the affected layers, directly improving their conditioning. Additional analysis on DINO~\cite{simeoni2025dinov3} is provided in the supplementary material.

\subsection{Quantization-Aware Training}

% \begin{table}
% \centering
% \caption{}
% % Note: The table is wider, so you may need to adjust '0.8\columnwidth' 
% % or remove the \resizebox entirely.
% \resizebox{0.9\columnwidth}{!}{
% \begin{tabular}{lccccc}
% \toprule
% Method & W3A4 & W4A4 & W5A5 & W6A6 \\
% \midrule
% AdamW & XXX\% & XXX\% & XXX\% & XXX\%  \\
% Adam$S^2D$ & XXX\% & XXX\% & XXX\% & XXX\%  \\
% \bottomrule
% \end{tabular}}
% \label{tab:w3a3_w4a4_qat} % Updated label to be more descriptive
% \end{table}

We combine $S^2D$ with QAT in a challenging low-bit regime. We implement W3A4 and W4A4 quantization, a harder setup where the impact of activation outliers is particularly acute. In low-bit settings, even modest range imbalance can cause catastrophic quantization errors, making outlier mitigation critical for maintaining accuracy.
We compare two QAT training regimes: a standard AdamW-based QAT baseline, and the same setup enhanced with $S^2D$ regularization. Both start from the pre-trained SigLIP2-Base-384 checkpoint and are trained for 10 epochs on ImageNet with simulated quantization in the forward pass and Straight Through Estimators (STEs) in backward propagation, using identical learning schedules and hyperparameter. We use symmetric per-channel quantization for weights and asymmetric per-tensor quantization for activations. The goal is to evaluate whether $S^2D$'s conditioning persists and provides benefits in a learnable quantization regime.
$S^2D$ provides substantial benefits: an absolute improvement of 2.5\% and 3.9\% for W3A4 and W4A4 respectively. The results presented in Figure \ref{fig:w3a3_w4a4_qat} show that $S^2D$'s conditioning can be combined with custom quantization learning schemes and is not limited to full precision fine-tuning regime.

\begin{figure}[t]
\centering
\includegraphics[width=0.75\linewidth]{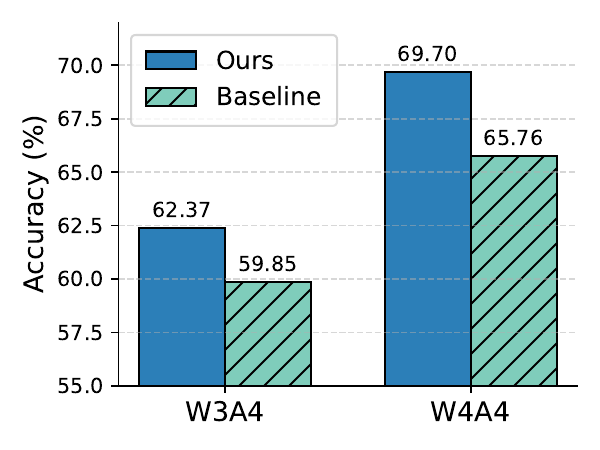}
\caption{\textbf{QAT Performance Gains.} Quantization-Aware Training (QAT) results for W3A4 and W4A4 on ImageNet1K using AdamW and 
AdamW+$S^2D$ (Ours). 
The bar plot shows that pairing QAT with $S^2D$ improves downstream accuracy over vanilla QAT..}
\label{fig:w3a3_w4a4_qat}
\end{figure}
\subsection{Downstream Task Adaptation}

\begin{table*}[t]
\centering
\caption{\textbf{Downstream Task Improvements.} Quantization results for object detection and instance segmentation on MS-COCO. 
The table reports ERQ (W4A4, W5A5, W6A6, W8A8) and PTQ4ViT (W6A6, W8A8) performance, showing that 
AdamW+$S^2D$ (Ours) delivers consistent performance gains across PTQ settings.}
\label{tab:downstream-tasks}
\resizebox{0.85\textwidth}{!}{
\begin{tabular}{l l c  cccc  cc}
\toprule
\multirow{2}{*}{\textbf{Task}} &
\multirow{2}{*}{\textbf{Method}} &
\multirow{2}{*}{\textbf{FP16}} &
\multicolumn{4}{c}{\textbf{ERQ}} &
\multicolumn{2}{c}{\textbf{PTQ4ViT}} \\
\cmidrule(lr){4-7}
\cmidrule(lr){8-9}
 & & &
\textbf{W4A4} & \textbf{W5A5} & \textbf{W6A6} & \textbf{W8A8} &
\textbf{W6A6} & \textbf{W8A8} \\
\midrule

\multirow{2}{*}{\textbf{Object Detection}} 
& AdamW & 50.1 & 1.8 & 10.8 & 45.8 & 48.2 & 2.2 & 42.1 \\
& Ours  & 50.3 & \textbf{16.6} & \textbf{40.7} & \textbf{47.8} & \textbf{49.7} & \textbf{24.7} & \textbf{48.2} \\
\midrule

\multirow{2}{*}{\textbf{Instance Segmentation}}
& AdamW & 43.6 & 0.50& 11.7 & 39.9 & 41.9 & 1.9 & 36.5 \\
& Ours  & 43.8 & \textbf{13.0} & \textbf{34.4} & \textbf{41.4} & \textbf{43.2} & \textbf{20.7} & \textbf{42.8} \\
\bottomrule
\end{tabular}
}
\end{table*}

% \begin{table*}[t]
% \centering
% \caption{Downstream Task Quantization Results: Object Detection and Instance Segmentation on MS-COCO dataset.}
% \label{tab:downstream-tasks}
% % Adjusted resizebox to \textwidth to accommodate new columns
% \resizebox{\textwidth}{!}{%
% % Added 3 'c's to the column specifier for the 3 new columns
% \begin{tabular}{lcccccccc}
% \toprule
% % Reordered columns for logical grouping by W/A and added new headers
% \textbf{Task} & \textbf{Method} & \textbf{FP16} & \textbf{ERQ-W4A4} & \textbf{ERQ-W5A5} & \textbf{PTQ4ViT-W6A6} & \textbf{ERQ-W6A6} & \textbf{PTQ4ViT-W8A8} & \textbf{ERQ-W8A8} \\
% \midrule
% % Added 'XXX' placeholders for new columns in their respective positions
% \multirow{2}{*}{\textbf{Object Detection}} & AdamW & 50.13  & 1.87 & 10.82 & 2.25 & 45.85 & 42.15 & 48.19 \\
% & Ours & 50.33 & \textbf{16.60} & \textbf{40.70} & \textbf{24.77} & \textbf{47.80} & \textbf{48.19} & \textbf{49.72} \\ 
% \midrule
% \multirow{2}{*}{\textbf{Instance Segmentation}} & AdamW & 43.64  & 0.45 & 11.69 & 1.90 & 39.94 & 36.52 & 41.91 \\
% & Ours & 43.84  & \textbf{13.00} & \textbf{34.38} & \textbf{20.67} & \textbf{41.45} & \textbf{42.77} & \textbf{43.22} \\ 
% \bottomrule
% \end{tabular}%
% }
% \end{table*}
\noindent\textbf{\textit{Object Detection and Instance Segmentation.} }To validate the performance of $S^2D$ on downstream vision tasks, we focus on object detection and instance segmentation, tasks that are fundamentally different from the pre-training objective and require significant model adaptation. We fine-tune the pre-trained SigLIP2-Base-384  backbone on MS-COCO \cite{lin2015microsoftcococommonobjects} dataset. For both tasks, we initialize the encoder from the pre-trained checkpoint and fine-tune using the AdamW baseline and AdamW combined with $S^2D$. We leverage Detectron2 \cite{wu2019detectron2} library and employ Generalized RCNN network with a FPN head. Fine-tuning is performed for 270K iterations until convergence and identical learning schedules and hyperparameters across both methods. After fine-tuning, the resulting models are quantized to W4A4 using ERQ and PTQ4ViT. Table~\ref{tab:downstream-tasks} presents the quantization performance on downstream tasks. Consistent with the ImageNet results, $S^2D$-fine-tuned models substantially outperform AdamW baselines across both tasks and quantization methods. On object detection, $S^2D$ achieves 29.88 points AP50 improvement with W5A5 ERQ quantization. Baseline PTQ4ViT approaches random performance at lower bits but $S^2D$ is able to retain relatively more performance across quantization bits. For semantic segmentation, the improvements are equally consistent demonstrating task-agnostic benefits. Importantly, full-precision accuracy is slightly better across both tasks, confirming that $S^2D$ functions purely as a better conditioning method without sacrificing model capacity.

\noindent\textbf{\textit{Vision-Language Models.} }
We evaluate $S^2D$ on LLaVA-1.5 \cite{liu2023improvedllava} training setting, combining a SigLIP2-Base-384 vision encoder with a Qwen2.5-0.5B language model \cite{yang2025qwen3}. This architecture presents a unique opportunity to study whether $S^2D$ can effectively condition heterogeneous model components with different pre-training dynamics and architectural properties.
We apply $S^2D$ regularization to both components of LLaVA 1.5 during fine-tuning. Following standard practice, we first pre-train the projector using Llava-Pretrain \cite{liu2023llava} dataset to align the vision encoder and language decoder. Post this, fine-tuning is performed on Llava-Instruct \cite{liu2023improvedllava} using identical hyperparameters across AdamW and Adam$S^2D$ approaches. Following fine-tuning, both full-precision and quantized versions are evaluated on standard VLM benchmarks including GQA \cite{hudson2019gqa}, TextVQA \cite{singh2019towards}, POPE \cite{li2023evaluating} and DocVQA \cite{mathew2021docvqa} covering visual question answering, fine-grained OCR and hallucination benchmarks.
Table~\ref{tab:vlm_results} demonstrates that $S^2D$ provides consistent improvements across diverse VLM evaluation benchmarks. It is interesting to note that the full-precision performance of $S^2D$ conditioned model is better than the baseline. Across the benchmarks $S^2D$ shows strong quantization performance except in the case of POPE where the differences are insignificant. The consistent improvements demonstrates that $S^2D$'s spectral conditioning benefits the entire VLM pipeline. The improvement in full-precision performance shows that $S^2D$ may serve as a better conditioning method for multi-modal models.

\begin{table}
\centering
\caption{\textbf{VLM Quantization.} Evaluation of LLaVA-1.5 (SigLIP2-Base-384 + Qwen2.5-0.5B) under AdamW and AdamW+$S^2D$ (Ours) fine-tuning. 
$S^2D$ provides gains in both full-precision and quantized settings across GQA, TextVQA, POPE, and DocVQA, 
demonstrating its effectiveness as a spectral conditioning method for multimodal VLM pipelines.}
\label{tab:vlm_results}
\resizebox{\columnwidth}{!}{%
\begin{tabular}{lccccc}
\toprule
\textbf{Benchmark} & \textbf{Method} & \textbf{FP16} & \textbf{W4A4} & \textbf{W5A5} & \textbf{W6A6} \\
\midrule
\multirow{2}{*}{\textbf{GQA}}     & AdamW & 47.6 & 35.3 & 38.5 & 43.9 \\
                                  & Ours  & 55.4 & 40.1 & 47.6 & 52.8 \\ 
\midrule
\multirow{2}{*}{\textbf{TextVQA}} & AdamW & 31.8 & 6.79 & 10.4 & 19.2\\
                                  & Ours  & 36.7 & 9.8 & 17.6 & 28.0\\
\midrule
\multirow{2}{*}{\textbf{POPE}} & AdamW & 69.8 & 66.6 & 67.4 & 69.8 \\
                               & Ours  & 70.3 & 66.6 & 66.6 & 69.4 \\
\midrule
\multirow{2}{*}{\textbf{DocVQA}} & AdamW & 14.3 & 5.7 & 6.0 & 8.8 \\
                                 & Ours  & 17.4 & 6.1 & 7.6 & 12.4 \\
\bottomrule
\end{tabular}%
}
\end{table}
\subsection{Latency analysis}
There is a natural increase in training time due to the SVD computations required by $S^2D$. In our setup, a full SVD pass over all layers takes approximately 18 seconds, while the corresponding gradient pass requires about 6 seconds on 8×NVIDIA A100 GPUs. Over 10 epochs of training (25K steps), we perform roughly 250 SVD updates. However, we can effectively hide this latency by parallelizing the SVD computation and triggering it 3 iterations before it is needed. As shown in Section~\ref{sec:implementation}, we can also approximate the SVD by computing it every 100 steps without any loss in performance, making this amortized strategy both feasible and efficient. As a result, the overall overhead introduced by $S^2D$ is negligible. 

\section{Conclusions and Future Work}

\noindent\textbf{\textit{Conclusions.} }This work addresses the emergence of activation outliers that arise from prolonged optimization with AdamW, establishing that these outliers are optimization artifacts rather than functionally meaningful features, and that their severity escalates with pre-training scale. To diagnose and quantify this phenomenon, we introduced PCDR$_k$, an effective spectral diagnostic, and proposed $S^2D$, a geometrically principled regularization method that selectively suppresses dominant singular components while preserving useful model capacity. Extensive experiments show that $S^2D$ yields substantial and consistent improvements across PTQ and QAT pipelines while maintaining full-precision accuracy across architectures and tasks. Together, our findings demonstrate that spectral conditioning is a powerful and general mechanism for mitigating activation outliers and enabling robust low-precision deployment.

% This work addresses activation outliers that emerge from prolonged optimization with adaptive methods like AdamW, establishing them as optimization artifacts that escalate with pre-training scale. To mitigate this, we introduces PCDR$_k$, a diagnostic tool, and Selective Spectral Decay ($S^2D$), a regularization method that targets pathological spectral components. Experiments show $S^2D$ substantially improves existing state of the art PTQ and QAT without sacrificing full-precision performance across a diverse range of vision tasks. 

\noindent\textbf{\textit{Future Work.} }There are several promising directions building on this work. First, exploring the interaction between $S^2D$ and alternative optimizers may reveal whether spectral conditioning remains necessary under fundamentally different optimization dynamics. Second, applying $S^2D$ during large-scale pre-training, rather than only during downstream fine-tuning, could suppress outlier formation at its source and potentially improve both stability and generalization. Finally, extending our analysis to multimodal architectures would help assess the universality of the spectral mechanisms uncovered here. 

{
    \small
    \bibliographystyle{ieeenat_fullname}
    \bibliography{main}
}

% WARNING: do not forget to delete the supplementary pages from your submission 
\clearpage
\setcounter{page}{1}
\appendix
\section*{Appendix}
%% ============================================================
\section{Theoretical Analysis: Spectral Bounds}
\label{app:spectral_analysis}

In this section, we provide the formal proof for Theorem \ref{thm:activation_bound} stated in the main text and expand upon the connection between spectral norms and the propagation of activation outliers in deep networks.

\subsection{Proof of Activation Magnitude Bound}

We first restate the bound regarding the relationship between the spectral norm of a weight matrix and the magnitude of the output activations.

\setcounter{theorem}{0} % Adjust this number to match the main paper reference if needed
\begin{theorem}[Restated]
\label{thm:activation_bound_app}
Let $\mathbf{y} = \mathbf{Wx}$ be the output of a linear layer for an input vector $\mathbf{x} \in \mathbb{R}^n$ and weight matrix $\mathbf{W} \in \mathbb{R}^{m \times n}$. The Euclidean norm of the output vector is bounded by the spectral norm of the weight matrix $\sigma_{\max}(\mathbf{W})$, such that:
\begin{equation}
    \|\mathbf{y}\|_2 \le \sigma_{\max}(\mathbf{W}) \cdot \|\mathbf{x}\|_2
\end{equation}
\end{theorem}

\begin{proof}
Let $\|\cdot\|_2$ denote the Euclidean norm on vectors. The matrix norm induced by the vector Euclidean norm (the spectral norm) is defined as:
\begin{equation}
    \|\mathbf{W}\|_2 := \sup_{\mathbf{x} \neq \mathbf{0}} \frac{\|\mathbf{Wx}\|_2}{\|\mathbf{x}\|_2} = \sigma_{\max}(\mathbf{W})
\end{equation}
where $\sigma_{\max}(\mathbf{W})$ is the largest singular value of $\mathbf{W}$. By the definition of the supremum, for any specific $\mathbf{x} \in \mathbb{R}^n$, it must hold that:
\begin{equation}
    \frac{\|\mathbf{Wx}\|_2}{\|\mathbf{x}\|_2} \le \sigma_{\max}(\mathbf{W})
\end{equation}
Multiplying both sides by $\|\mathbf{x}\|_2$ (assuming $\mathbf{x} \neq \mathbf{0}$; the trivial case holds for $\mathbf{x}=\mathbf{0}$) yields:
\begin{equation}
    \|\mathbf{y}\|_2 = \|\mathbf{Wx}\|_2 \le \sigma_{\max}(\mathbf{W})\|\mathbf{x}\|_2
\end{equation}
\end{proof}

This result establishes that a large spectral norm is a necessary condition for a linear layer to amplify a reasonably scaled input into a large-magnitude output outlier.

%% ============================================================
\section{Algorithm}
\label{app:algorithm}

Algorithm~\ref{alg:s2d_simplified_clean} presents the complete training procedure for Selective Spectral Decay (S$^2$D). The algorithm operates by periodically computing singular value decompositions and selectively penalizing dominant spectral components responsible for activation outliers.

\textbf{Key steps:}
\begin{enumerate}[nosep]
    \item \textbf{Periodic SVD updates (Lines 6--17):} Every $k$ training steps, the algorithm computes the SVD of each layer's weight matrix $\mathbf{W}^{(l)} = \mathbf{U}\mathbf{\Sigma}\mathbf{V}^{\top}$, identifying the spectral structure of the weights.

    \item \textbf{Outlier detection (Line 9):} Using the Principal Component Dominance Ratio (PCDR), the algorithm identifies the minimum rank $k$ where $\text{PCDR} \geq \tau$. This determines how many dominant singular values are responsible for creating outliers.

    \item \textbf{Penalty matrix construction (Lines 11--12):} For layers with identified outliers, a penalty matrix $\mathbf{G}_{\text{reg}}$ is constructed by raising the top-$k$ singular values to power $n$ and reconstructing the partial matrix. This targets only the problematic spectral components.

    \item \textbf{Gradient update (Lines 18--21):} During each training step, the standard task gradient is augmented with the cached penalty $\lambda\mathbf{G}_{\text{reg}}$, applying selective regularization pressure to the weight components aligned with the largest singular values while leaving other components largely unaffected.
\end{enumerate}

\begin{algorithm*}[t]
\caption{Selective Spectral Decay (S$^2$D)}
\label{alg:s2d_simplified_clean}
\begin{algorithmic}[1]
    \State \textbf{Input:} Weights $\mathbf{W}$
    \State \textbf{Hyperparams:} Power $n$, Reg.\ strength $\lambda$, Learning rate $\eta$, Update frequency $k$, PCDR threshold $\tau$

    \State Initialize step counter $t \leftarrow 0$
    \State Initialize penalty matrices $\mathbf{G}_{\text{reg}}^{(l)} \leftarrow \mathbf{0}$ for all layers $l$

    \While{\textbf{training}}
        \If{$t \bmod k = 0$} \Comment{Periodic spectral update}
            \For{layer $l = 1$ \textbf{to} $L$}
                \State $\mathbf{U}, \mathbf{\Sigma}, \mathbf{V} \leftarrow \text{SVD}(\mathbf{W}^{(l)})$
                \State $\hat{k} \leftarrow \min \{ k' : \text{PCDR}(\mathbf{\Sigma}, k') \ge \tau \}$ \Comment{Identify outlier rank cutoff}
                \If{$\hat{k}$ is defined}
                    \State $\mathbf{\Sigma}_n \leftarrow \text{diag}(\sigma_1^n, \dots, \sigma_{\hat{k}}^n)$
                    \State $\mathbf{G}_{\text{reg}}^{(l)} \leftarrow \mathbf{U}_{:, 1:\hat{k}} \, \mathbf{\Sigma}_n \, (\mathbf{V}_{:, 1:\hat{k}})^\top$ \Comment{Cache penalty matrix}
                \Else
                    \State $\mathbf{G}_{\text{reg}}^{(l)} \leftarrow \mathbf{0}$ \Comment{No significant outlier rank found}
                \EndIf
            \EndFor
        \EndIf

        \State $\nabla_{\mathbf{W}} \mathcal{L}_{\text{task}} \leftarrow \text{Backward}(\mathcal{L}_{\text{task}}(\text{batch}))$ \Comment{Standard task loss gradient}

        \For{layer $l = 1$ \textbf{to} $L$} \Comment{Apply regularized update}
            \State $\mathbf{W}^{(l)} \leftarrow \mathbf{W}^{(l)} - \eta \left( \nabla_{\mathbf{W}^{(l)}} \mathcal{L}_{\text{task}} + \lambda \, \mathbf{G}_{\text{reg}}^{(l)} \right)$
        \EndFor
        \State $t \leftarrow t + 1$
    \EndWhile
\end{algorithmic}
\end{algorithm*}

%% ============================================================
\section{Analysis of Outlier Origins}
\label{app:outlier_origins}

In this section, we expand upon the connection between adaptive optimizers and the formation of activation outliers in transformer models. While the dominant singular \textit{directions} ($U, V$) encode semantically meaningful representations, their \textit{extreme magnitudes} ($\Sigma$) are predominantly optimization artifacts rather than functionally necessary features.

\paragraph{Evidence from Prior Work.}
Several independent lines of evidence support this characterization.
\cite{caples2024adam} show that Adam-trained models exhibit rapid growth in excess kurtosis ($>100$), indicating the emergence of significant outlier channels, whereas SGD-trained models maintain substantially lower kurtosis throughout training.
\cite{elhage2023privileged} demonstrate that Adam's component-wise normalization privileges the training basis; when this basis is rotated to decorrelate the model, outliers disappear without performance loss, confirming they are not functionally necessary.
Furthermore, \cite{he2024understanding} link outlier features to large diagonal adaptive learning rates in Adam, showing that reducing adaptivity minimizes outlier formation.

\paragraph{Implications for S$^2$D.}
These findings establish that outlier magnitudes are preventable artifacts of AdamW's basis preference and anisotropic update dynamics. S$^2$D acts as a targeted counter-force to this spectral amplification. The fact that S$^2$D maintains or improves full-precision accuracy (e.g., +1.2\% on LLaVA, Table 5 in the main text) confirms that suppressing these extreme magnitudes is benign to the model's semantic capacity.

%% ============================================================
\section{Comparison with Alternative Regularization Approaches}
\label{app:regularization_comparison}

\paragraph{S$^2$D vs.\ Rotation Methods (SpinQuant, QuIP).}
Rotation-based methods mitigate outliers by \textit{redistributing} activation magnitudes through a learned or analytically computed basis transformation $W' = RW$. In contrast, S$^2$D \textit{suppresses} the spectral cause directly by penalizing the dominant singular values in $\Sigma$. These two strategies are thus orthogonal and potentially complementary. Additionally, S$^2$D avoids the online inference overhead of rotation methods, producing standard weights compatible with vanilla deployment kernels.

\paragraph{S$^2$D vs.\ Standard Spectral Regularization.}
Standard spectral regularization applies uniform pressure to every singular component across the network. As shown in Table~\ref{tab:pcdr_ablation_supp}, applying spectral regularization without the PCDR diagnostic collapses W4A4 accuracy to 40.1\% on ImageNet (SigLIP2-Base-384). PCDR acts as a surgical guide, targeting only the spectral components identified to cause pathological activation concentration. Without this selectivity, regularization indiscriminately suppresses both harmful and beneficial spectral components, degrading model capacity.

\begin{table}[t]
    \centering
    \caption{\textbf{Impact of PCDR-guided layer selection.} Comparison of S$^2$D with and without PCDR targeting on ImageNet classification using ERQ quantization (SigLIP2-Base-384). Without PCDR, uniform spectral regularization severely degrades low-bit performance.}
    \label{tab:pcdr_ablation_supp}
    \begin{tabular}{lcc}
    \toprule
     & No PCDR & S$^2$D \\
    \midrule
    W4A4 & 40.1 & \textbf{73.0} \\
    W5A5 & 77.1 & \textbf{81.9} \\
    W6A6 & 81.2 & \textbf{82.7} \\
    \bottomrule
    \end{tabular}
\end{table}

%% ============================================================
\section{Self-Supervised Backbone}
\label{app:self_supervised}

We extend our experiments to DINOv3~\cite{simeoni2025dinov3}, a recently trained self-supervised vision backbone. We attempted to quantize DINOv3 using the official ERQ codebase~\cite{zhong2025accurateposttrainingquantizationvision}; however, all ERQ configurations yielded near-random accuracy regardless of bit-width, suggesting an incompatibility with the self-supervised feature distribution. We therefore report results exclusively with PTQ4ViT~\cite{yuan2024ptq4vitposttrainingquantizationvision} in Table~\ref{tab:dino_ptq}. We also present spectral statistics for DINOv3 in Table~\ref{tab:pcd_and_singular_values_comparison}. Both the PTQ results and spectral patterns are consistent with the SigLIP2 findings, confirming that S$^2$D improves quantization of modern self-supervised models.

\begin{table*}[t]
\centering
\caption{\textbf{Comparison of FFN activation and weight statistics for SigLIP2 and DINOv3.}
We report the PCDR$_1$ and maximum absolute activation of the FFN layers, along with the maximum singular value of their corresponding weights, after fine-tuning with AdamW and AdamW+S$^2$D.}
\label{tab:pcd_and_singular_values_comparison}
\resizebox{0.8\linewidth}{!}{
\begin{tabular}{llcccccc}
\toprule
& & \multicolumn{2}{c}{\textbf{PCDR$_{1}$}} & \multicolumn{2}{c}{\textbf{Max Abs.\ Activation}} & \multicolumn{2}{c}{\textbf{Max Singular Value}} \\
\cmidrule(lr){3-4} \cmidrule(lr){5-6} \cmidrule(lr){7-8}
\textbf{Model} & \textbf{Layer} & \textbf{AdamW} & \textbf{S$^2$D} & \textbf{AdamW} & \textbf{S$^2$D} & \textbf{AdamW} & \textbf{S$^2$D} \\
\midrule
\multirow{2}{*}{SigLIP2} & Layer 5 & 0.91 & 0.46 & 176.40 & 59.68 & 10.38 & 3.22 \\
& Layer 9 & 0.77 & 0.09 & 1166.2 & 614.7 & 7.87 & 3.96 \\
\midrule
\multirow{2}{*}{DINOv3} & Layer 2 & 0.47 & 0.11 & 1048.41 & 440.38 & 37.29 & 17.94 \\
& Layer 4 & 0.45 & 0.22 & 115.55 & 67.56 & 8.67 & 3.79 \\
\bottomrule
\end{tabular}
}
\end{table*}

\begin{table}[t]
    \centering
    \caption{\textbf{PTQ4ViT quantization results on DINOv3-Base} with and without S$^2$D regularization. ImageNet top-1 accuracy (\%) is reported.}
    \label{tab:dino_ptq}
    \begin{tabular}{lccc}
        \toprule
        \textbf{Method} & \textbf{W8A8} & \textbf{W6A6} & \textbf{W5A5} \\
        \midrule
        AdamW & 58.03 & \textbf{57.04} & 28.52 \\
        AdamW + S$^2$D  & \textbf{76.53} & 56.10 & \textbf{30.69} \\
        \bottomrule
    \end{tabular}
\end{table}

%% ============================================================
\section{Extension to Language Models}
\label{app:language_models}

To evaluate the generality of S$^2$D beyond vision and vision-language tasks, we conduct a preliminary experiment on a pure language model. We fine-tune Qwen2.5-0.5B using supervised fine-tuning (SFT) on the Dolci dataset, applying S$^2$D with the same hyperparameters used in the vision experiments (no task-specific tuning). We then evaluate on GSM8K (0-shot) under round-to-nearest (RTN) quantization at various bit-widths. Results are presented in Table~\ref{tab:gsm8k_zeroshot_supp}.

Despite using fewer than 1B training tokens and no hyperparameter adaptation for the language domain, S$^2$D consistently improves quantized performance across W8A8, W7A7, and W6A6 settings, with gains of +2.2, +2.6, and +2.0 percentage points respectively. The slight reduction in full-precision performance ($-1.0$) reflects the regularization trade-off, which is more than compensated by the quantization gains. We expect that a learning rate sweep and longer training schedule would further improve both full-precision and quantized results.

\begin{table}[t]
\centering
\caption{\textbf{GSM8K 0-shot results with RTN quantization for Qwen2.5-0.5B.} S$^2$D improves quantized accuracy across all tested bit-widths using the same hyperparameters from the vision experiments.}
\label{tab:gsm8k_zeroshot_supp}
\begin{tabular}{lcc}
\toprule
 & \textbf{S$^2$D} & \textbf{Baseline} \\
\midrule
FP & 28.2 & 29.2 \\
W8A8 & \textbf{26.2} & 24.0 \\
W7A7 & \textbf{24.7} & 22.1 \\
W6A6 & \textbf{19.8} & 17.8 \\
\bottomrule
\end{tabular}
\end{table}

%% ============================================================
\section{Hyperparameter Sensitivity}
\label{app:hyperparameters}

We analyze the robustness of S$^2$D by varying its key hyperparameters. As shown in Table~\ref{tab:hparam_sensitivity_cols}, the default configuration ($k{=}100$, $\text{topk}{=}3$, $\text{threshold}{=}0.95$) yields the highest W4A4 performance at 73.0\%.

We observe that a larger update interval ($k{=}100$) outperforms frequent updates ($k{=}10$), suggesting that accumulating statistics over a longer horizon improves stability. Interestingly, increasing $\text{topk}$ from 3 to 10 results in a marginal performance drop, indicating that outlier mitigation is most effective when targeting only the few most dominant singular directions. Finally, a stricter PCDR threshold of 0.95 proves optimal compared to lower values.

\begin{table}[t]
\centering
\caption{\textbf{Hyperparameter sensitivity analysis for S$^2$D on ImageNet.}
We vary the SVD computation frequency ($k$), number of targeted singular values ($\text{topk}$), and PCDR threshold. The \textbf{default} configuration is shown in bold.}
\label{tab:hparam_sensitivity_cols}
\begin{tabular}{ccccc}
\toprule
$\boldsymbol{k}$ & \textbf{topk} & $\boldsymbol{\tau}$ & \textbf{FP16 (\%)} & \textbf{W4A4 (\%)} \\
\midrule
\multicolumn{5}{c}{\textit{Vary SVD update frequency $k$}} \\
\midrule
10 & 3 & 0.95 & 85.0 & 72.2 \\
\textbf{100} & \textbf{3} & \textbf{0.95} & \textbf{85.0} & \textbf{73.0} \\
\midrule
\multicolumn{5}{c}{\textit{Vary number of targeted singular values}} \\
\midrule
100 & 5 & 0.95 & 85.0 & 72.5 \\
100 & 10 & 0.95 & 84.9 & 72.6 \\
\midrule
\multicolumn{5}{c}{\textit{Vary PCDR threshold $\tau$}} \\
\midrule
100 & 3 & 0.90 & 84.9 & 72.4 \\
100 & 3 & 0.80 & 84.8 & 72.8 \\
\bottomrule
\end{tabular}
\end{table}

\subsection{Sensitivity to Power Exponent $n$}

The power exponent $n$ in S$^2$D controls the degree of non-uniformity in the penalty applied to singular values: larger $n$ concentrates regularization pressure more aggressively on the dominant singular values. We chose $n{=}2$ (yielding a $\sigma^3$ penalty) to exert stronger regularization on the singular values contributing to outliers while preserving smaller components. Table~\ref{tab:n_sensitivity_supp} confirms that while higher orders ($n{=}3, 4$) still outperform the baseline, $n{=}2$ offers the optimal trade-off between outlier suppression and capacity preservation.

\begin{table}[t]
\centering
\caption{\textbf{Sensitivity to power exponent $n$.} ImageNet accuracy (\%) using ERQ quantization on SigLIP2-Base-384. $n{=}2$ provides the best balance across bit-widths.}
\label{tab:n_sensitivity_supp}
\begin{tabular}{lcccc}
\toprule
 & Baseline & $\boldsymbol{n{=}2}$ & $n{=}3$ & $n{=}4$ \\
\midrule
W4A4 & 65.6 & \textbf{73.0} & 68.3 & 69.6 \\
W5A5 & 78.5 & \textbf{81.9} & 79.0 & 78.7 \\
W6A6 & 81.1 & \textbf{82.7} & 81.7 & 82.4 \\
\bottomrule
\end{tabular}
\end{table}

\subsection{Amortized SVD Stability}

A potential concern with the amortized SVD computation (every $m{=}100$ steps) is whether the cached singular vectors ($U, V$) become stale and lead to inaccurate gradient updates. To validate this design choice, we analyzed the stability of the S$^2$D gradient penalty computed using cached versus freshly computed SVD factors. The cosine similarity between the two gradient signals remains above $0.99$ over the $m{=}100$ step caching interval, confirming that the singular vector subspaces evolve slowly relative to the caching frequency. This justifies the computational amortization and explains why $k{=}100$ outperforms more frequent updates ($k{=}10$) in Table~\ref{tab:hparam_sensitivity_cols}: the additional noise from frequent recomputation slightly destabilizes training without meaningful accuracy benefit.

\end{document}